\documentclass[letterpaper]{article}
\usepackage{aaai20}
\usepackage{times}
\usepackage{helvet}
\usepackage{courier}
\usepackage[hyphens]{url}
\usepackage{graphicx}

\usepackage{graphicx}
\usepackage{subcaption}
\usepackage{xspace}
\usepackage{amssymb}
\usepackage{amsthm}
\usepackage{amsmath}
\usepackage{booktabs}
\usepackage{multirow}
\usepackage{array}

\newtheorem{definition}{Definition}
\newtheorem{theorem}{Theorem}

\newcommand{\deftitle}[1]{\textbf{#1}\xspace}
\newcommand{\Problem}{\ensuremath{\mathcal{P}}}
\newcommand{\eqbydef}{\ensuremath{\doteq}}

\newcommand{\tuple}[1]{\ensuremath{\langle #1 \rangle}\xspace}

\newcommand{\blexec}[1]{\textsc{BlEx}(#1)\xspace}
\newcommand{\dreexec}{\textsc{DREEx}\xspace}

\usepackage{algorithm}
\usepackage[noend]{algpseudocode}
\usepackage{algorithmicx}

\algrenewcommand\algorithmicindent{1em}%
\algrenewcommand\alglinenumber[1]{\tiny #1:}

\algnewcommand\algorithmicforeach{\textbf{for each}}
\algnewcommand{\IfThenElse}[3]{
  \State \algorithmicif\ #1\ \algorithmicthen\ #2\ \algorithmicelse\ #3}
\algdef{S}[FOR]{ForEach}[1]{\algorithmicforeach\ #1\ \algorithmicdo}

\algnewcommand{\MyIf}[2]{
  \State \algorithmicif\ #1\ \algorithmicthen\ #2}
\algblockdefx[MyElse]{MyElse}{EndMyElse}{\algorithmicelse}{\algorithmicend\ \algorithmicif}

\makeatletter
\ifthenelse{\equal{\ALG@noend}{t}}%
  {\algtext*{EndMyElse}}
  {}%
\makeatother

\MakeRobust{\Call}
\usepackage{etoolbox}

\usepackage{cleveref}
\crefname{figure}{figure}{figures}
\Crefname{figure}{Figure}{Figures}

\usepackage{tikz}
\usetikzlibrary{positioning, shapes, calc}

\usepackage{todonotes}

\newcommand{\myparagraph}[1]{\vskip 4pt \noindent \textbf{#1. }}

\title{Towards Efficient Anytime Computation and Execution of Decoupled Robustness Envelopes for Temporal Plans}

\author{Michael Cashmore\textsuperscript{1},
Alessandro Cimatti\textsuperscript{2},
Daniele Magazzeni\textsuperscript{3},
Andrea Micheli\textsuperscript{2},
Parisa Zehtabi\textsuperscript{3}\\
\textsuperscript{1}University of Strathclyde, United Kingdom, \{michael.cashmore\}@strath.ac.uk, \\
\textsuperscript{2}Fondazione Bruno Kessler, Italy, \{cimatti, amicheli\}@fbk.eu, \\
\textsuperscript{3}King's College London, United Kingdom, \{daniele.magazzeni, parisa.zehtabi\}@kcl.ac.uk\\
}

\begin{document}

\maketitle

\begin{abstract}
One of the major limitations for the employment of model-based planning and scheduling in practical applications is the need of costly re-planning when an incongruence between the observed reality and the formal model is encountered during execution. Robustness Envelopes characterize the set of possible contingencies that a plan is able to address without re-planning, but their exact computation is extremely expensive; furthermore, general robustness envelopes are not amenable for efficient execution.

In this paper, we present a novel, anytime algorithm to approximate Robustness Envelopes, making them scalable and executable. This is proven by an experimental analysis showing the efficiency of the algorithm, and by a concrete case study where  the execution of robustness envelopes significantly reduces the number of re-plannings.
\end{abstract}

\section{Introduction}

When planning and scheduling techniques are employed in practical applications, one of the major problems is the need for on-line re-planning when the observed contingencies are not aligned with the ones that were considered at planning time.
These situations are common, because it is arguably impossible to predict the entire range of situations an autonomous system can encounter, especially when the planning domain encompasses time and temporal constraints.
Unfortunately, re-planning can be costly in terms of time, and computational resources can be scarce on-board, so limiting the use of re-planning is very important for practical purposes.
In principle, it is also possible to continue with the execution of a plan even when the observed contingencies are unexpected, optimistically hoping for a successful completion.
However, this approach offers no formal guarantee, and is prone to the risk of continuing execution of a plan that is bound to fail.
Several approaches have been proposed in the literature to address this problem
(see~\cite{ing17} for a survey focused on robotics).
Some authors propose to post-process plans and generalize them relying on the scheduling constraints that are relevant for execution~\cite{policella-flexibility,muscettola-envelope,frank-linear-envelope}.
Another line of research focuses on the creation of ``least commitment plans'', i.e. plans that are left partially open by the planner so that the execution can be adapted to some variation in the contingencies~\cite{ixtet,rmpl,europa,apsi,platinum}.
Others tackled the idea of transforming temporal plans with no adaptability into flexible plans~\cite{flexibility-do}.
Finally, one can explicitly model the uncertainties in the planning problem and construct a plan that offers formal guarantees with respect to such a model. Examples include Conformant and Contingent Planning~\cite{traverso-book}, Probabilistic Planning~\cite{probabilistic-planning} and Strong Temporal Planning with Uncertain Durations~\cite{micheli_aij} that considers temporal uncertainty in the durations of actions.
Recently, \emph{Robustness Envelopes} (REs) have been proposed to overcome several limitations of the approaches mentioned above. REs formally capture the possible contingencies that a given temporal plan, obtained by planning in a \emph{deterministic domain}, can deal with, without having to re-plan \cite{robustness-envelopes}.
REs are regions defined over a set of numeric parameters that represent possible contingencies, and contain all the parameter valuations ensuring plan validity.
In general, REs may be non-convex, and can express dependencies between the parameters.
%
%
However, the technique proposed in~\cite{robustness-envelopes} has two main drawbacks limiting its practical applicability.
First, the exact computation of REs is extremely expensive: the proposed approach is doubly exponential in the size of the planning problem. Second, REs in their general form are not suited for efficient execution: the dependencies among parameters might require run-time reasoning.

%
In this paper, we overcome these limitations, achieving scalability and executability.
%
%
We focus on \emph{Decoupled Robustness Envelopes} (DREs), i.e. hyper-rectangular REs where the dependencies among parameters are not present, and are thus much easier to execute.
Our first contribution is a novel and scalable algorithm for computing DREs as sound approximations of REs.
The algorithm is anytime, and proceeds by incrementally under-approximating the RE with increasingly large DREs.
The algorithm can be stopped at any time, providing a meaningful result already amenable to start execution.
In its general formulation, the RE for a given plan is naturally modeled as a quantified first order formula in the theory of Linear Real Arithmetic. Our algorithm does not need to precisely compute the quantifier-free description of the RE (which requires an expensive step of quantifier elimination, and is ultimately responsible for the inefficiency demonstrated in~\cite{robustness-envelopes}). Rather, it starts from a degenerate DRE consisting of a single point, and progressively tries to enlarge it along different dimensions, checking if each extension is contained in the RE, until a given precision is reached. The algorithm relies on \emph{quantifier-free} queries to a Satisfiability Modulo Theory~\cite{smt} solver.

Our second contribution is to demonstrate the practical use of DREs in a robotic executor, extending the classical flow from planning to execution to re-planning, as follows. First, a plan is generated from a deterministic model using temporal planning technologies, and transformed into a Simple Temporal Network (STN) formulation~\cite{dechter-stn}; at this point, we parametrize the durations of some of the actions in the plan and/or the consumption rates in the domain specification. DREs are then computed for the introduced parameters and passed to the executor. In turn, the dispatching of the actions in the STN plan begins and continues until one observed duration or consumption rate happens to be outside of the DRE. At this point, the executor detects that the plan is no longer guaranteed to succeed, and re-planning is triggered.

The proposed approach was implemented in the ROSPlan~\cite{rosplan} framework, and experimentally evaluated along two directions.
The algorithm for DRE generation was compared against the base line in~\cite{robustness-envelopes}, demonstrating orders-of-magnitude improvements compared to the exact computation of REs, and the ability to deal with a much larger number of parameters.
The overall execution loop has been evaluated on a family of concrete case studies in a logistic domain, showing that the use of DREs,
compared to the optimistic executor in ROSPlan, significantly reduces the number of re-plannings and improves the execution success-rate.

\section{Background}

%
We consider planning problems expressed in the PDDL 2.1~\cite{pddl21} temporal planning language; for the sake of brevity we do not report the full syntax and semantics of such planning problems, but we directly introduce the parametrized planning problem idea adapted from~\cite{robustness-envelopes}.
\begin{definition}
  A \deftitle{parametrized planning problem} $\Problem_\Gamma$ is a tuple
  $\tuple{\Gamma, \Problem}$, where $\Gamma$ is a finite set of real-valued
  parameters $\{\gamma_1, \cdots, \gamma_n\}$ and $\Problem$ is a PDDL 2.1 planning
  problem in which conditions, effects, goals and initial states can
  contain parameters.
\end{definition}
\noindent
Intuitively, symbols (from a known set $\Gamma$) can be used in expressions where real-typed constants are usually allowed.

As customary in many cases of plan execution, we use plans expressed as Simple Temporal Networks (STN) \cite{dechter-stn}. An STN plan is a set of constraints of the form $t_i - t_j \le k$ where $t_i$ and $t_j$ are time points linked to action happenings (i.e. either the start or the end of an action instance in the plan) and $k \in \mathbb{Q}$.
In addition, we allow parameters in the plan specification by generalizing the notion of STN plans.
\begin{definition}
  A \deftitle{parametrized STN plan} $\pi_\Gamma$ for a parametrized planning problem $\Problem_\Gamma \eqbydef \tuple{\Gamma, \Problem}$ is an STN Plan where some constraints are in the form $t_i - t_j \le \gamma_i$ where $t_i$ and $t_j$ are time-points of the STN plan and $\gamma_i \in \Gamma$.
\end{definition}
We define the Robustness Envelope (RE) for a parametrized problem and plan as the set of possible values for the parameters that make the plan valid when the symbols are substituted with values in the plan and problem specifications.
In order to compute the RE, \citeauthor{robustness-envelopes} define a set of logical formulae that characterize the RE and use quantifier elimination techniques (e.g. \cite{fourier-motzkin}) to explicitly construct the region.
The encoding is divided in three expressions: indicated as $enc_{tn}^{\pi_\Gamma}$, $enc_{\mathit{eff}}^{\pi_\Gamma}$ and $enc_{\mathit{proofs}}^{\pi_\Gamma}$.
The formula $enc_{tn}^{\pi_\Gamma}$ encodes the temporal constraints imposed by $\pi_\Gamma$ limiting the possible orderings of time points. The formula $enc_{\mathit{eff}}^{\pi_\Gamma}$ encodes the effects of each time point on the state variables, while $enc_{\mathit{proofs}}^{\pi_\Gamma}$ encodes the validity properties of the plan, namely that the conditions of each executed action are satisfied, that the goal is reached, and that the $\epsilon$-separation constraint imposed by PDDL 2.1 is respected.
Then, let $\bar{X}$ be the set of all the variables appearing in the formulae above except the parameter values, the RE is characterized by all the models of the following formula.
$$\exists \bar{X} . (enc_{tn}^{\pi_\Gamma} \wedge enc_{\mathit{eff}}^{\pi_\Gamma} ) \wedge \forall \bar{X} . ((enc_{tn}^{\pi_\Gamma} \wedge enc_{\mathit{eff}}^{\pi_\Gamma}) \rightarrow enc_{\mathit{proofs}}^{\pi_\Gamma})$$

As observed by \citeauthor{robustness-envelopes}, any under-approximation of the RE gives sound information on the contingencies in which the plan is guaranteed to be valid; in particular, a convenient restriction for the representation and handling of REs is to associate a closed interval of possible values to each parameter, defining an hyper-rectangle. If such hyper-rectangle is contained in the RE, we have a "Decoupled Robustness Envelope" (DRE) that retains the guarantees of the RE but avoids the complexity of inter-dependencies among parameters.
\begin{definition}\label{def:dre}
  A \deftitle{Decoupled Robustness Envelope} for a parametrized planning problem $\Problem_\Gamma$ and STN plan $\pi_\Gamma$ is a bound assignment $\rho : \Gamma \rightarrow \mathbb{Q}_{>=0} \times \mathbb{Q}_{>=0}$, such that any parameter assignment $v : \Gamma \rightarrow \mathbb{Q}_{>=0}$, with $l \le v(\gamma) \le u$ and $\tuple{l, u} \eqbydef \rho(\gamma)$, is contained in the RE for $\Problem_\Gamma$ and $\pi_\Gamma$.
\end{definition}
\noindent
Note that many DREs are possible for a given problem and plan: it suffices that all the assignments allowed by the DRE are points in the RE. In this paper, we elaborate on this idea and propose an algorithm that incrementally builds DREs that are contained within the unconstrained RE without paying the cost of explicitly computing the RE itself.

Finally, we highlight how given any two DREs $\rho_1$ and $\rho_2$ three cases are possible: either $\rho_1$ is subsumed by $\rho_2$ (i.e. for each parameter $\gamma$, $\rho_1(\gamma) \subseteq \rho_2(\gamma)$), or $\rho_1$ subsumes $\rho_2$, or the two DREs are incomparable. Hence, there is no absolute best DRE in general: we aim for a DRE that is not subsumed by any other, but there can be multiple DREs with this property.

\section{Execution Flow}\label{sec:exe}

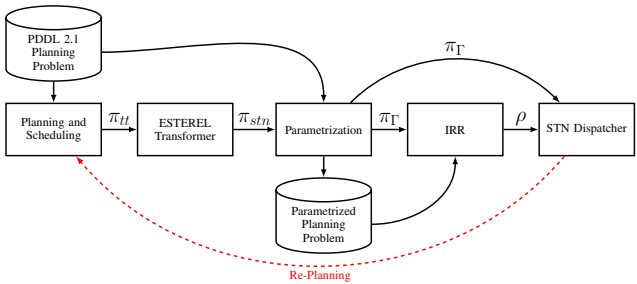
\begin{figure}[tb]
    \centering
    \resizebox{\columnwidth}{!}{\begin{tikzpicture}
  [
    data/.style = {cylinder, shape border rotate=90, draw=black, minimum height=1.5cm, minimum width=2.5cm, shape aspect=.25, very thick, align=center, text width=2.5cm},
    tool/.style = {rectangle, draw=black, thick, minimum height=1.5cm, minimum width=2.5cm, very thick,  align=center, text width=2.5cm},
    link/.style = {-latex, very thick}
  ]

  \node[data] (model) {PDDL 2.1 Planning Problem};
  \node[tool] (planning) [below=.6cm of model] {Planning and Scheduling};
  \node[tool] (tostn) [right=1cm of planning] {ESTEREL Transformer};
  \node[tool] (parametrization) [right=1.2cm of tostn] {Parametrization};
  \node[data] (pmodel) [below=.6cm of parametrization] {Parametrized Planning Problem};
  \node[tool] (irr) [right=1cm of parametrization] {IRR};
  \node[tool] (exec) [right=1cm of irr] {STN Dispatcher};

  \draw[link] (model) -- (planning);
  \draw[link] (planning) -- node[midway,above] {\huge $\pi_{tt}$} (tostn);
  \draw[link] (tostn) -- node[midway,above] {\huge $\pi_{stn}$} (parametrization);
  \draw[link] (model) to [out=0,in=90] (parametrization);
  \draw[link] (parametrization) -- node[midway,above] {\huge $\pi_\Gamma$} (irr);
  \draw[link] (parametrization) to [out=45,in=135] node[midway,above] {\huge $\pi_\Gamma$} (exec);
  \draw[link] (parametrization) -- (pmodel);
  \draw[link] (pmodel) to [out=0,in=270] (irr);
  \draw[link] (irr) -- node[midway,above] {\huge $\rho$} (exec);
  \draw[link, color=red, dashed] (exec) to [out=230,in=310] node[midway,below] {Re-Planning} (planning);
\end{tikzpicture}}
    \caption{Overview of the proposed flow.}
    \label{fig:flow}
\end{figure}

The general idea we pursue in this paper is to exploit the information and the generalization provided by the synthesis of REs to limit the number of re-plannings and increase the success-rate in execution. In particular, we propose the flow from planning to execution depicted in figure \ref{fig:flow}. Starting from a planning problem formulation expressed in PDDL 2.1, we use any off-the-shelf temporal planner\footnote{Several existing PDDL planners are unable to generate flexible STNs either because of an implementation limitation or because the technique does not allow it (e.g. SAT-based planners). Our approach is able to generate DREs from these planners as well, and work in concert with existing algorithms for the execution of STNs.} to compute a timed sequence of actions that reaches the goal from the initial state. We call this plan ``time-triggered'' (indicated with $\pi_{tt}$) in the picture. This plan is not natively amenable for execution because it defines one specific trace that does not allow any adaptability: it is extremely unlikely for a real system to be perfectly controlled to satisfy a specific trace. Hence, $\pi_{tt}$ needs to be converted in a flexible, executable STN ($\pi_{stn}$) by the ESTEREL transformer of ROSPlan.
The usual flow would pass this STN directly to the dispatcher for translating the plan actions into commands for the robotic platform at the proper time. Instead, here we pre-process this plan using REs in the hope of generalizing its applicability and reducing the number of re-plannings. In particular, the STN plan is passed to a parametrization component that re-reads the planning problem formulation and enriches it with parameters, generating a Parametric Planning Problem and a parametric STN plan ($\pi_\Gamma$). Those are the inputs for the computation of the RE. In our flow, for performance reasons and to avoid complex run-time reasoning, instead of computing the exact, unconstrained RE, we use a novel algorithm, called Incremental Rectangular-Robustification (IRR for short), that computes a DRE. The algorithm is anytime, so that it is possible to retrieve unfinished computations and exploit them in execution: in fact, any under-approximation of the final result retains the needed properties of the RE.
At this point, we pass the DRE ($\rho$) together with the parametrized STN plan to the STN dispatcher. We modified the dispatching algorithm to exploit the information in the DRE to limit the re-plannings to situations where they are needed. In particular, the dispatcher translates the actions, while checking that the observed values for the parameters (being either action durations, resources or rates) fall within the bounds imposed by $\rho$. If this is not the case, re-planning is needed and the whole flow is re-executed.

\myparagraph{Parametrization}
The first non-standard step highlighted in figure \ref{fig:flow} is the parametrization. In fact, there are multiple ways in which parameters can be added to a deterministic temporal planning problem to characterize useful quantities for execution. In general, one can parametrize any numeric quantity in the planing problem whose value might differ from the environment in which the plan will be executed. In order to be useful for the STN dispatcher, however, such quantities must be eventually observable (directly or indirectly). Otherwise, it is impossible for the executor to check whether the RE is still satisfied or if a re-planning is needed. In this paper experimentation, we focused on two such quantities, namely the durations of actions and resource consumption rates. The former is a classical source of uncertainty when temporal planning is employed in a robotics scenario, the latter is another source of uncertainty that can perturbate the execution of a plan, for example when the resource harvesting is not fully controllable (e.g. a solar panel yield depends on the weather) or when the consumption is not fully predictable (e.g. the battery consumption is very hard to precisely estimate as it depends on temperature, exact capacity and so on).

\section{Incremental Rectangular-Robustification}

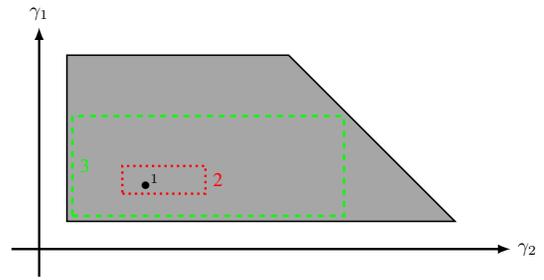
\begin{figure}[tb]
    \centering
    \resizebox{.85\columnwidth}{!}{\begin{tikzpicture}
  [
    axis/.style = {-latex, very thick},
    region/.style = {fill=gray!70, draw=black, thick},
  ]

  \draw[axis] (11.5, 5.0) -- (11.5, 9.5) node [above] {$\gamma_1$};
  \draw[axis] (11.0, 5.5) -- (20.0, 5.5) node [right] {$\gamma_2$};

  \draw[region] (12.0, 6.0) -- (19.0, 6.0) -- (16.0, 9.0) -- (12.0, 9.0) -- cycle;

  \node at (13.5, 6.7) {\textbullet$^1$};
 
  \draw[red, very thick, dotted] (13, 6.5) -- (14.5, 6.5) -- node [midway,right] {2} (14.5, 7) -- (13, 7) -- cycle;
 
  \draw[green, very thick, dashed] (12.1, 6.1) -- (17.0, 6.1) -- (17.0, 7.9) -- (12.1, 7.9) -- node [midway,right] {3} cycle;

\end{tikzpicture}}
    \caption{A graphical representation of IRR: starting from the parameter values from the original plan (depicted as the black point), IRR tries to construct increasingly better under-approximations (the colored rectangles) of the RE (the gray area), without actually computing it. Upon termination, each edge of the resulting DRE is guaranteed to be at most $\beta$ apart from the border of the actual region.}
    \label{fig:algo-intuition}
\end{figure}

We now present our novel algorithm for incrementally computing decoupled robustness envelopes. We call this algorithm "Incremental Rectangular-Robustification" (IRR).

The idea behind the algorithm is to construct incrementally better hyper-rectangular under-approximations of the RE for a given problem and plan. In fact, this constitutes a direct way of computing a DRE by generate-and-test. The starting point is the de-generated hyper-rectangle composed of the single point given by the parameter values of the original plan. The algorithm tries to extend the hyper-rectangle along one dimension (i.e. it tries to widen the interval of possibilities associated to one of the parameters) and checks if the resulting hyper-rectangle is in fact an under-approximation of the RE. If it is, the new hyper-rectangle is kept as it is guaranteed to be a valid DRE. Otherwise, another dimension or another increment is chosen for the algorithm to proceed. The general intuition behind the algorithm is depicted in figure  \ref{fig:algo-intuition}.

\begin{algorithm}[tb]
  \small
  \begin{algorithmic}[1]
    \State{$enc_{valid} \gets $ \Call{QuantifierElimination}{$\exists \bar{X} . enc_{tn}^{\pi_\Gamma} \wedge enc_{\mathit{eff}}^{\pi_\Gamma}$}}
    \vskip 3pt
    \Function{IRR}{$\beta:\mathbb{Q}_{>0}$}
    \State{$R \gets \{\gamma \rightarrow [\pi(\gamma), \pi(\gamma)] \mid \gamma \in \Gamma\}$}
    \State{$\Delta \gets \{\gamma \rightarrow max(\pi(\gamma) \times \omega_{\gamma}, \beta) \mid \gamma \in \Gamma \}$}
    \State{$\Theta \gets \{\gamma \rightarrow \{\texttt{UB}, \texttt{LB}\} \mid \gamma \in \Gamma \}$}
  	\While{$\exists \gamma \in \Gamma . \Delta(\gamma) \ge \beta$}
  	  \State{$\tilde{\gamma}\gets $ \Call{Pick}{$\{\gamma \mid \gamma \in \Gamma \wedge \Theta(\gamma) \not = \emptyset \wedge \Delta(\gamma) \ge \beta\}$}}
  	  \State{$\theta \gets $ \Call{Pick}{$\Theta(\tilde{\gamma})$}}
      \State{$[l, u] \gets R(\tilde{\gamma})$}
  	  \IfThenElse{$\theta = \texttt{UB}$}{$u \gets (u + \Delta(\tilde{\gamma}))$}{$l \gets (l - \Delta(\tilde{\gamma}))$}
  	  \State{$R' \gets \{\gamma \rightarrow R(\gamma) \mid \gamma \in \Gamma \wedge \gamma \not = \tilde{\gamma}\} \cup \{\tilde{\gamma} \rightarrow [l, u]\}$}
  	  \MyIf{\Call{CheckInEnvelope}{$R'$}}{$R \gets R'$}
  	  \MyElse
  	    \State{$\Theta(\tilde{\gamma}) \gets \Theta(\tilde{\gamma}) \setminus \theta$}
  	    \If{$\Theta(\tilde{\gamma}) = \emptyset$}
  	      \State{$\Delta(\tilde{\gamma}) \gets \Delta(\tilde{\gamma}) / 2$;  \ \ $\Theta(\tilde{\gamma}) \gets \{\texttt{LB}, \texttt{UB}\}$}
  	    \EndIf
  	  \EndMyElse
  	\EndWhile
  	\State{\Return{$R$}}
  	\EndFunction
  	\vskip 3pt
  	\Function{CheckInEnvelope}{$R$}
  	  \State{$enc_R \gets \bigwedge_{\gamma \in \Gamma, [l, u] = R(\gamma)} l \le \bar{\gamma} \wedge \bar{\gamma} \le u$}
  	  \MyIf{\Call{IsSAT}{$enc_R \wedge \neg enc_{valid}$}}{\Return{false}}
  	  \MyElse
  	    \State{\Return{\Call{IsValid}{$(enc_{tn}^{\pi_\Gamma} \wedge enc_{\mathit{eff}}^{\pi_\Gamma} \wedge enc_R) \rightarrow enc_{\mathit{proofs}}^{\pi_\Gamma}$}}}
  	  \EndMyElse
 	\EndFunction
  \end{algorithmic}
  \caption{\label{algo:irr} Incremental Rectangular-Robustification}
\end{algorithm}

Algorithm \ref{algo:irr} reports the pseudo-code of IRR. The formula $enc_{valid}$ is computed once and off-line. It corresponds to the basic requirements for the hyper-rectangle to be a valid DRE: only parameter values that are not contradicting the STN plan and the causal flow of effects are admissible. This is the same as the first piece of the logical formulation in~\cite{robustness-envelopes}, but luckily it is the easier part of the quantification and can be efficiently computed.
Then, the \textsc{IRR} function is in charge of computing a hyper-rectangle $R$ maintaining the following invariant: at each step, $R$ is a subset of the RE of the problem. The hyper-rectangle $R$ is represented as a pair of bounds (lower- and upper-) assigned to each parameter (this directly models a DRE as per definition \ref{def:dre}), and is initialized (line 3) with the values of the non-parametric plan $\pi$. The algorithm uses two functions to control how the hyper-rectangle is transformed from one cycle to the next. $\Delta$ associates to each parameter a number that is the value used to increase the upper-bound or to decrease the lower-bound for that parameter.
The initial value for $\Delta$ is the original value of the parameter scaled by a weight for such parameter, but any positive number bigger than $\beta$ is enough to guarantee soundness and termination of the algorithm. Note that these weights can be used to express preferences on the parameters: a higher weight pushes the algorithm to expand a specific parameter more than others.
The function $\Theta$ is used to decide in which direction the interval of a parameter can be extended. Two directions are possible: \texttt{UB} indicates that we want to extend the upper-bound and \texttt{LB} indicates that we want to decrease the lower-bound (line 10). Initially both directions are possible, but when we discover (line 13) that one direction is infeasible with the current $\Delta$, we remove this direction from the possibilities. The value of $\Delta$ gets refined to converge to a value lower than $\beta$, so each time $\Delta$ gets updated, we reset $\Theta$ to allow both directions once again.

The main loop of the algorithm continues until all the values of $\Delta$ are lower than $\beta$: this is to guarantee that the minimum distance from each border of the hyper-rectangle and the border of the actual RE is at most $\beta$. The algorithm picks a parameter $\tilde{\gamma}$ to be analyzed among the parameters having at least one direction available in $\Theta$ and that have not converged already (line 7); then, it generates a candidate hyper-rectangle $R'$ by extending either the lower- or the upper- bound of $\tilde{\gamma}$. At this point, we check if $R'$ is contained in the RE or not (line 12). If it is, we keep it and continue the loop, otherwise, we discard this hyper-rectangle and we record that with the current $\Delta$ we cannot extend $\tilde{\gamma}$ in this direction by removing the direction $\theta$ from $\Theta(\tilde{\gamma})$. Moreover, if no direction is left for $\tilde{\gamma}$, we halve its value of $\Delta$ and reset $\Theta$ so that $\tilde{\gamma}$ can be tentatively extended again using a smaller step (lines 15-16).

The core part of the algorithm consists in checking a candidate hyper-rectangle for containment in the actual RE, without explicitly computing the region itself. This is done via the \textsc{CheckInEnvelope} function that performs two SMT checks corresponding to the two quantifiers appearing in the RE logical formulation of \cite{robustness-envelopes}. The first check looks for points belonging to $R$ that are not parts of the validity region $enc_{valid}$, the second checks if the rectangle (together with the guarantees from the plan and the effects) implies the proof requirements characterizing the REs. The important point here, is that both checks are quantifier-free, i.e. no quantifier elimination is involved.
\begin{theorem}
   The \textsc{CheckInEnvelope}($R$) function returns $true$ if and only if $R$ is a valid DRE
\end{theorem}
\begin{proof}
    The algorithm logically checks the following formula: $\neg (\exists \bar{\Gamma} . enc_R \wedge \neg enc_{valid}) \wedge \forall \bar{\Gamma}, \bar{X} . (enc_{tn}^{\pi_\Gamma} \wedge enc_{\mathit{eff}}^{\pi_\Gamma} \wedge enc_R) \rightarrow enc_{\mathit{proofs}}^{\pi_\Gamma}$, that can be rewritten as $\forall \bar{\Gamma} . enc_R \rightarrow (enc_{valid} \wedge (\forall \bar{X} . (enc_{tn}^{\pi_\Gamma} \wedge enc_{\mathit{eff}}^{\pi_\Gamma}) \rightarrow enc_{\mathit{proofs}}))$ that states that $enc_R$ is a subset of the encoding of the RE. Then, for definition \ref{def:dre}, R is the encoding of a valid DRE.
\end{proof}

An interesting feature of the algorithm is that it is ``anytime'', i.e. at each time, we can take the hyper-rectangle $R$ and we have the guarantee that $R$ is contained in the RE and is thus a valid DRE.
Moreover, the algorithm is guaranteed to terminate if the RE is finite in all dimensions.
\begin{theorem}
    If the the robustness envelope is bounded in all dimensions, \textsc{IRR} always terminates.
\end{theorem}
\begin{proof}
    All the values in $\Delta$ are initially positive and whenever the candidate rectangle is found to exit the RE (line 13) one of the values in $\Delta$ is halved. Eventually all the parameters will be considered and they will be eventually found to exit the RE because it is bounded in all dimensions. Therefore, all the values of $\Delta$ will become smaller than $\beta$.
\end{proof}

We highlight that \textsc{IRR} is in fact an optimization procedure that incrementally maximizes the size of a starting DRE, terminating when a maximal DRE is found within the given precision limit $\beta$.

\begin{figure*}[t]
  \centering
  \begin{subfigure}[b]{.325\textwidth}
    \includegraphics[width=\textwidth]{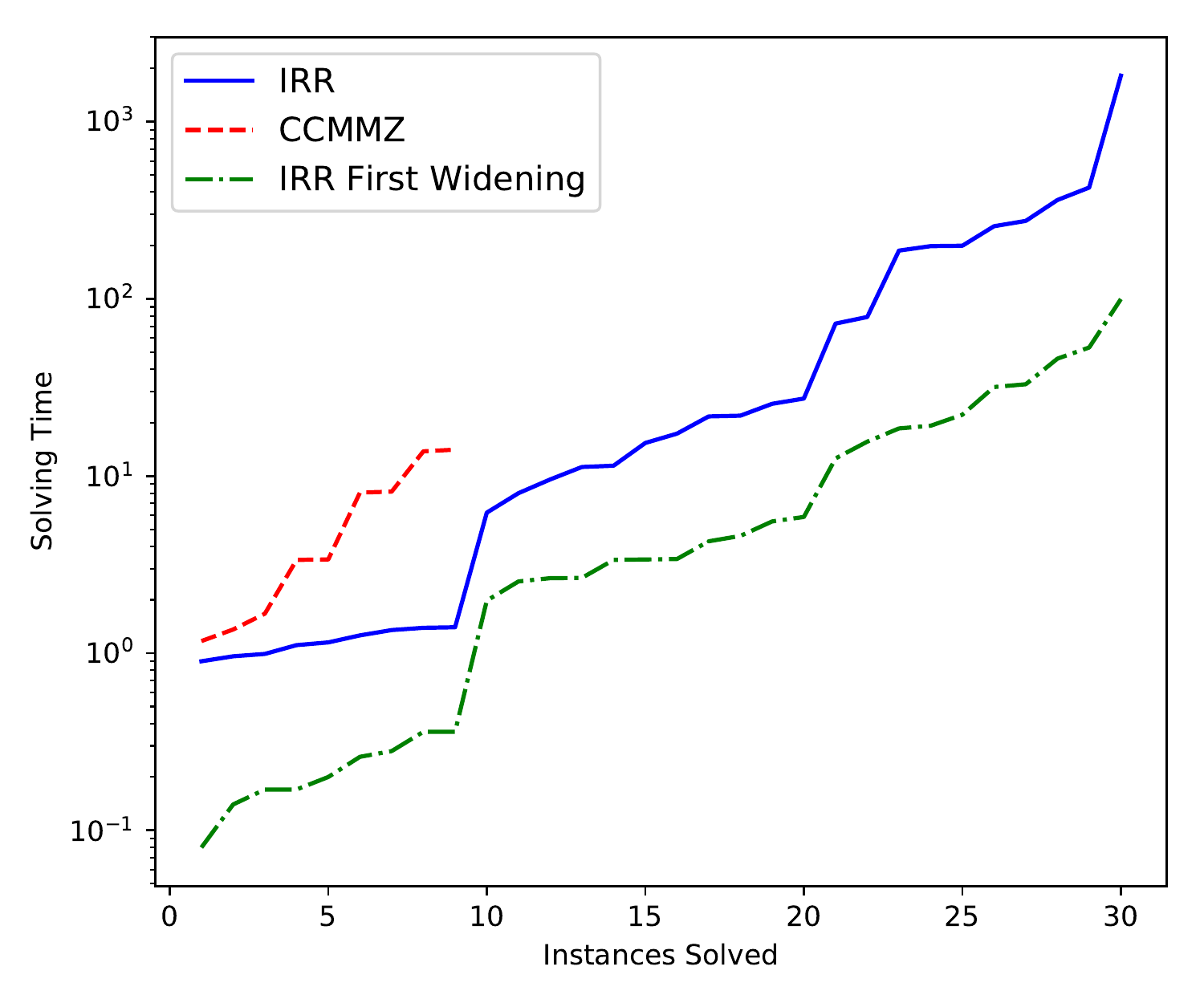}
    \caption{\label{fig:auv}AUV Domain.}
  \end{subfigure}
  \hfill
  \begin{subfigure}[b]{.325\textwidth}
    \includegraphics[width=\textwidth]{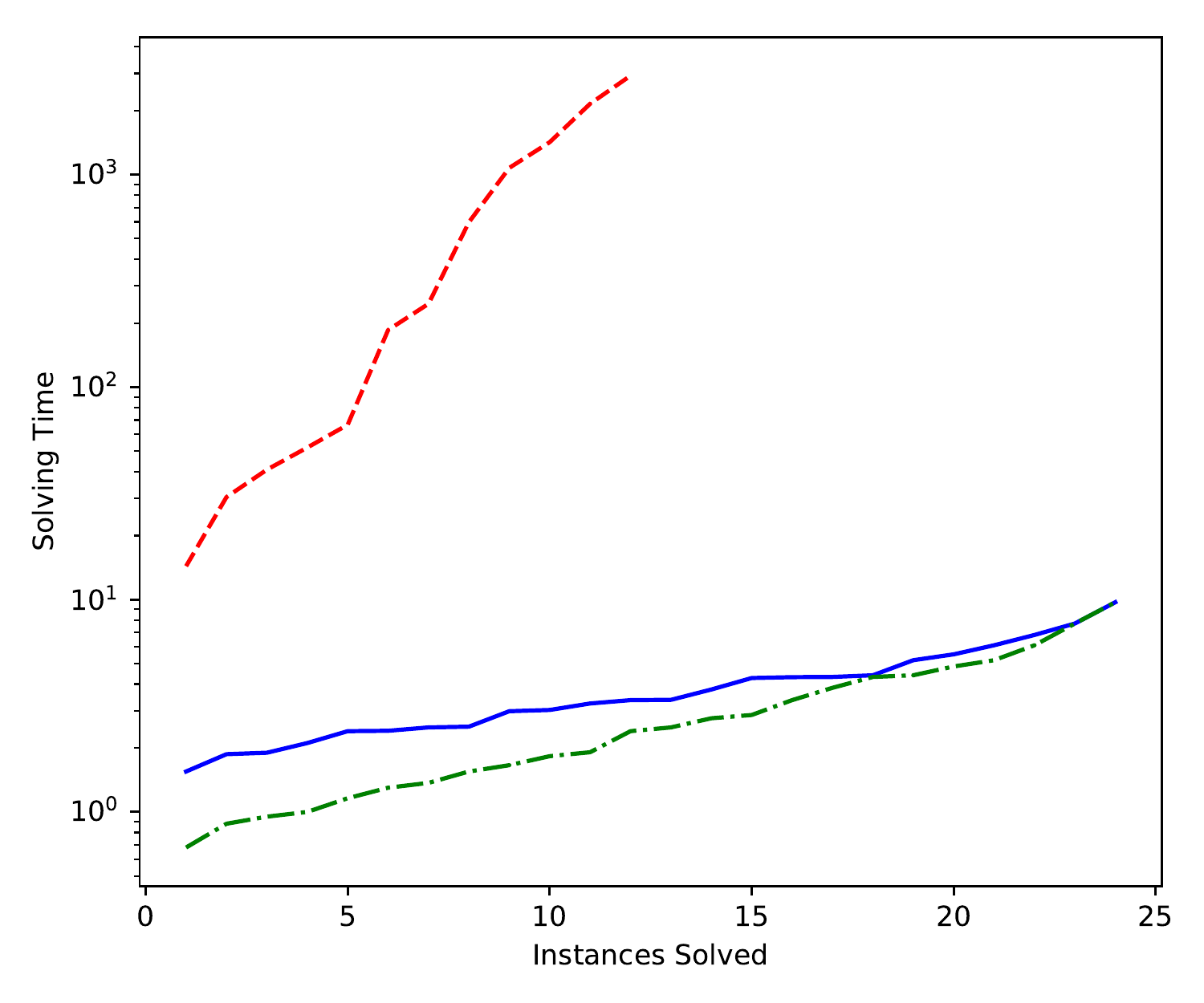}
    \caption{\label{fig:solar-rover}Solar-Rover Domain.}
  \end{subfigure}
  \hfill
  \begin{subfigure}[b]{.325\textwidth}
    \includegraphics[width=\textwidth]{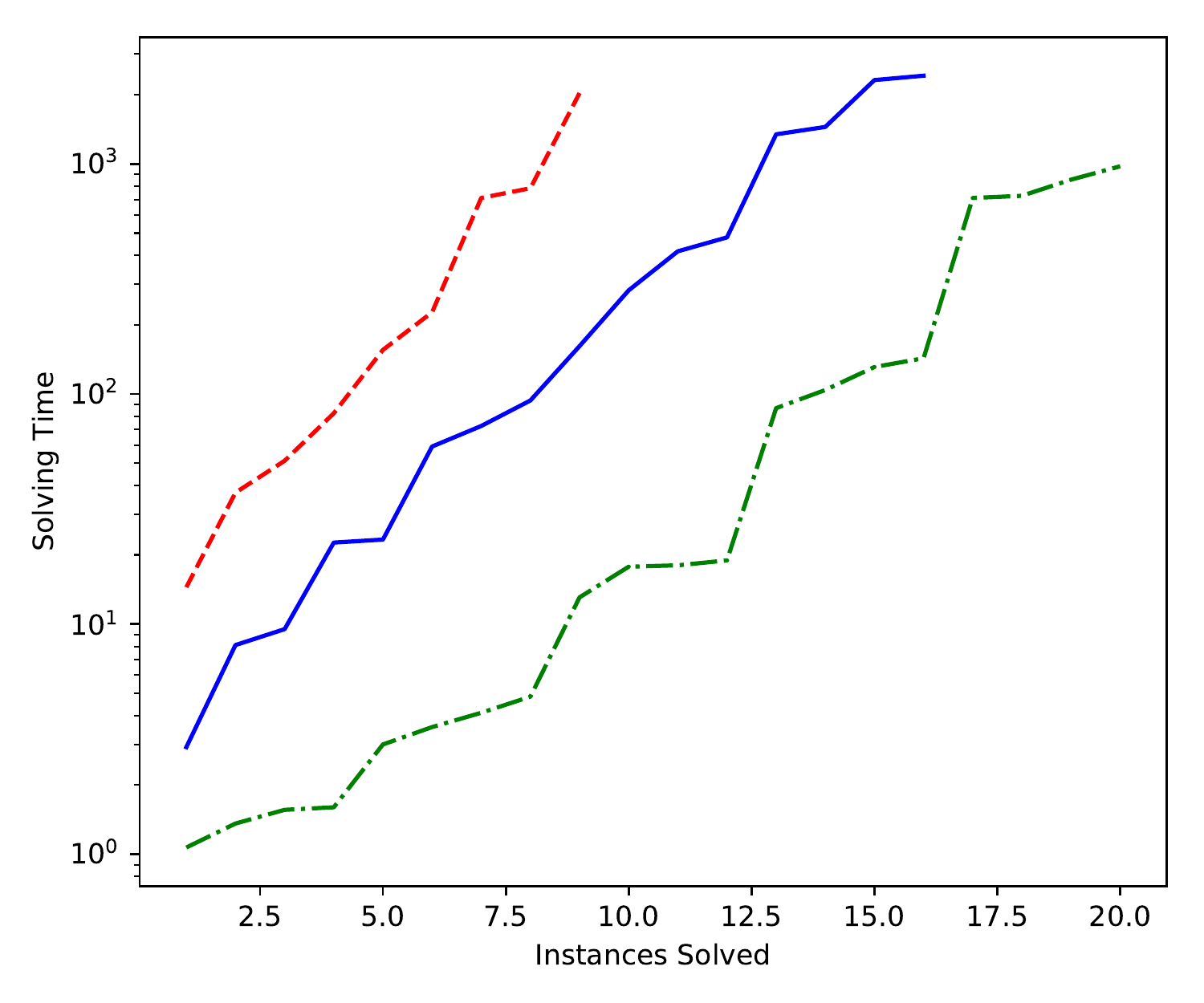}
    \caption{\label{fig:generator}Generator Linear Domain.}
  \end{subfigure}
  \caption{\label{fig:domains} Scalability experiments on [Cashmore \emph{et al.}, 2019] domains: the number of solved instances (sorted by difficulty for each solver) is considered on the X axis and is compared with the logarithmic time needed to solve each instance (lower, longer lines are better).}
\end{figure*}

\section{Experiments}

We now present our empirical analysis that comprises three sets of experiments. The first aims at showing the superior performance of \textsc{IRR} as compared with the logical approach of~\cite{robustness-envelopes}. The second shows a practical use-case of the execution flow proposed in this paper when the duration of actions is uncertain. In the third experimentation we use our DRE technique to execute plans when the consumption rates of resources is uncertain.

\myparagraph{IRR}
We start by considering the experimental dataset and the tool (hereafter called \textsc{CCMMZ}) provided in~\cite{robustness-envelopes}.
The benchmarks use a varying number of parameters; in particular, AUV ranges from 1 to 8 parameters, Generator Linear from 1 to 4 and  Solar Rover between 1 and 4.
We compare our IRR implementation with \textsc{CCMMZ} on all the available instances and domains, measuring the total run-time and using the ``decoupled envelope generation'' functionality of the tool. Moreover, in order to take into account the anytime nature of IRR, we also measure the time at which the rectangle $R$ in IRR widens and becomes different than a single point (i.e. we measure the first time the algorithm \ref{algo:irr} reaches line 17) and we call this timing ``IRR First Widening''. In all our experiments we set $\beta = 1$ and all $\omega_i = 1$ to find the decoupled region approximated to a single unit with no preferences among the parameters (obviously, we set the same parameter preference also in \textsc{CCMMZ}). We executed all of the instances on a Xeon E5-2620 2.10GHz machine setting a time limit of 3600s and a RAM memory limit of 20GB.

Figure \ref{fig:domains} shows the result of this analysis. IRR is able to solve many more instances than \textsc{CCMMZ} and is consistently quicker. Moreover, we note how the first widening is often encountered quite early in the execution, marking the margin for anytime exploitation of IRR. In fact, after the first widening, IRR already computed a meaningful and non-trivial under-approximation of the RE that can be used for execution. This is particularly evident in the Generator Linear domain where the algorithm is unable to fully terminate in some cases, but the first widening point is reached.

In addition to these domains, we also experiment with several instances of a service-robot domain that we also use for the following execution experimental analysis. The domain, called ``Robot Delivery'' is a simplified version\footnote{A simplified RCLL domain was used because the PDDL provided in the RCLL image is not complete and the RCLL simulation requires external processes, e.g. a referee box. We are interested in the flexible execution success rate, so we created PDDL instances encoding logistics problems without any external processes.} of the domain used in the Planning and Execution Competition for Logistics Robots in Simulation~\cite{rcll}. The domain comprises a fleet of small robots that can navigate in an euclidean graph. These robots are tasked to pick and deliver orders within a deadline. Collecting orders requires two robots present at a machine. We scaled the number of parameters in the instances between 1 and 33.
\begin{figure}[tb]
  \centering
  \includegraphics[width=.75\columnwidth]{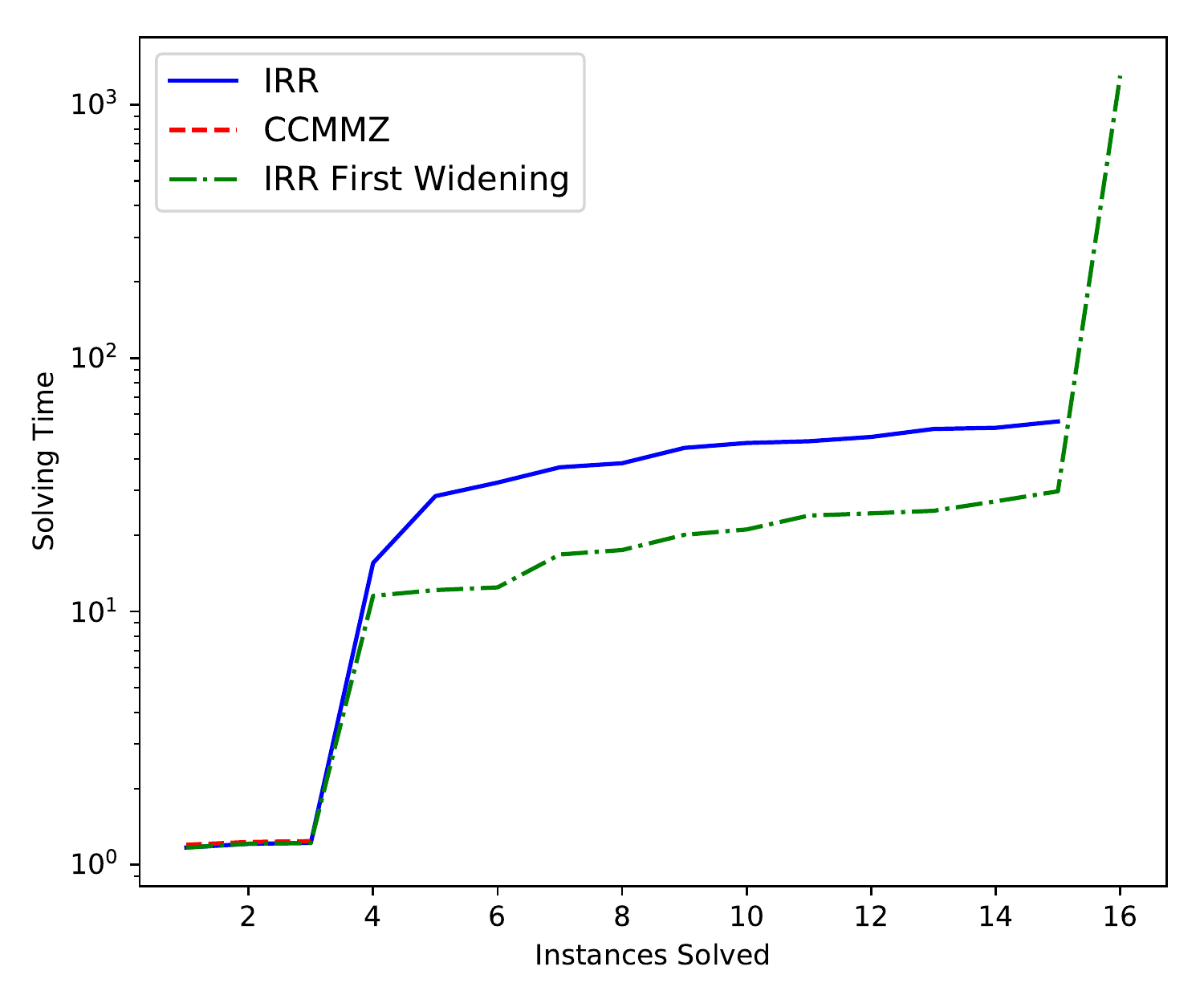}
  \caption{\label{fig:delivery} Scalability experiments on the delivery domain: the number of solved instances (sorted by difficulty for each solver) is considered on the X axis and is compared with the logarithmic time needed to solve each instance (lower, longer lines are better).}
\end{figure}
Figure \ref{fig:delivery} shows the scalability of IRR and \textsc{CCMMZ} on this domain. These instances are much harder for both the solvers compared to the previous domains; in fact, \textsc{CCMMZ} is only able to solve 3 instances, while IRR is able to solve 15 of them. Also in this case, the anytime nature of IRR is evident by observing the difference from the first widening and the algorithm completion.

\begin{figure}[tb]
  \centering
  \includegraphics[width=.75\columnwidth]{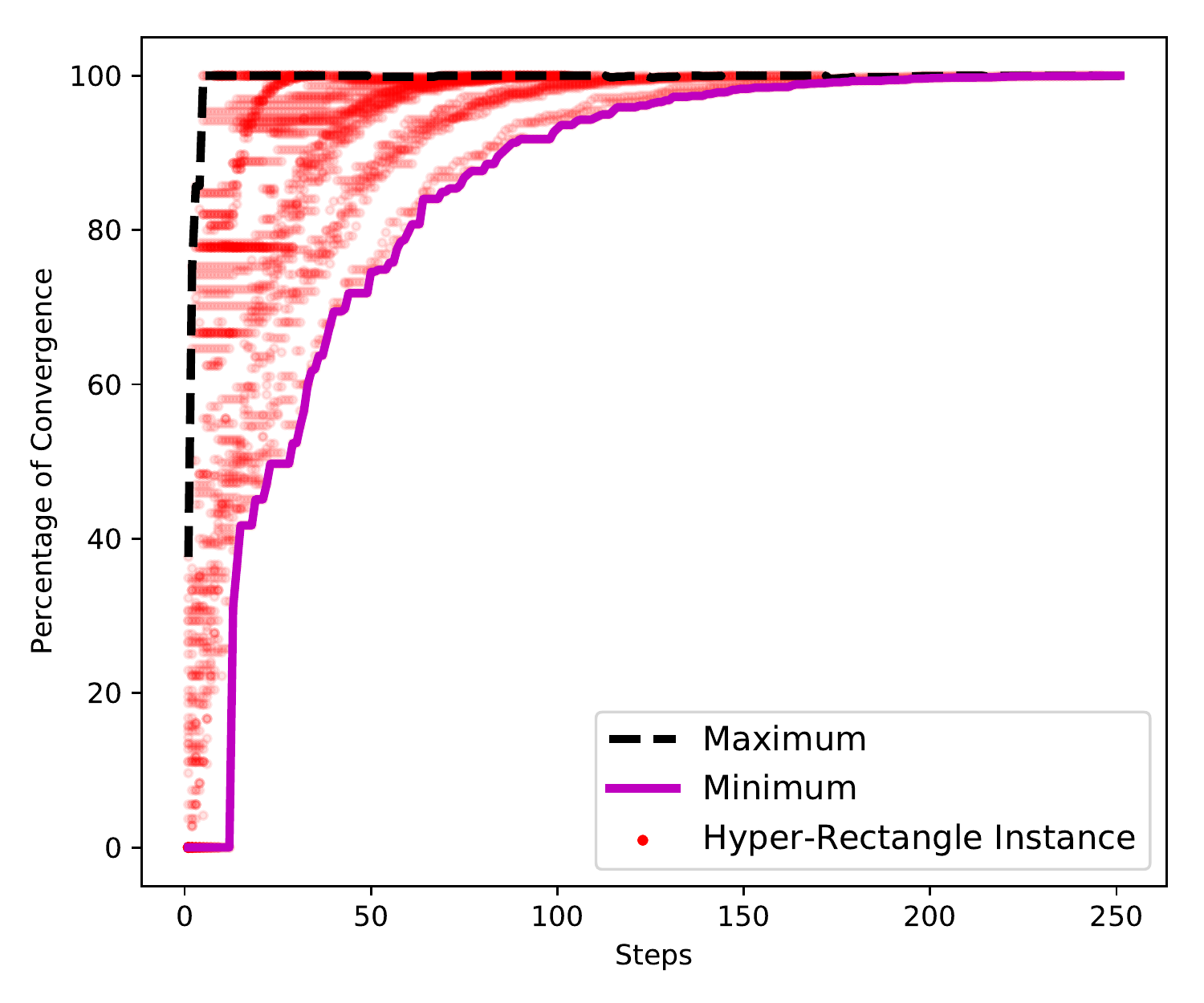}
  \caption{\label{fig:convergence} Convergence of IRR in terms of steps: each red dot is a DRR computed by IRR and the plot shows the progression in terms of convergence at each step of the algorithm. The purple line indicates the poorest convergence percentage for each step in any experiment; similarly the black, dashed line shows the best convergence.}
\end{figure}

Finally, we investigate how quickly the IRR algorithm converges in our experiments. We define convergence at step $i$ in a run of IRR that terminates with hyper-rectangle $R_{end}$ as follows ($R_i$ indicates the hyper-rectangle at step $i$).
$$
\resizebox{.75\columnwidth}{!}{$Convergence(i) = \frac{\sum_{[l, u] \in R_i} u-l}{\sum_{[l, u] \in R_{end}} u-l} \times 100$}
$$
\noindent
Intuitively, this gives the percentage of the region covered by $R_i$ with respect to $R_{end}$. (Note that $R_{end}$ contains $R_i$ because the IRR algorithm only expands previous hyper-rectangles.)
Figure \ref{fig:convergence} shows, for all problems solved by IRR in our benchmark set, the percentage of convergence achieved after any number of steps of the IRR algorithm. Clearly from the plot, in a limited number of steps we often approximate very well the final intervals; in particular, within 50 steps we already cover more than 70\% of the final sum of the interval sizes in all the cases.

\begin{table*}[tb]
  \centering
  \resizebox{.9\textwidth}{!}{
  \setlength{\tabcolsep}{2pt}
\begin{tabular}{l|cc|cc|cc|cc|cc}
  \multirow{2}{*}{\bf Executor} & \multicolumn{2}{c|}{\bf 1 Parameter} & \multicolumn{2}{c|}{\bf 2 Parameters} & \multicolumn{2}{c|}{\bf 3 Parameters} & \multicolumn{2}{c|}{\bf 4 Parameters} & \multicolumn{2}{c}{\bf 5 Parameters}\\
  & \bf \small Coverage  & \bf \small Avg Replans & \bf \small Coverage  & \bf \small Avg Replans & \bf \small Coverage  & \bf \small Avg Replans & \bf \small Coverage  & \bf \small Avg Replans & \bf \small Coverage  & \bf \small Avg Replans\\
  \hline &&&&&&&&&&\\[-1.5ex]
  \dreexec & \bf 92.2\% & \bf 0.1& \bf 85.8\% & \bf 0.2& \bf 83.2\% & \bf 0.2& \bf 72.8\% & \bf 0.1& \bf 63.0\% & \bf 0.1 \\
  \blexec{0} & 0.0\% & NA& 0.0\% & NA& 0.0\% & NA& 0.0\% & NA& 0.0\% & NA \\
  \blexec{10} & 24.5\% & 1.0& 4.8\% & 1.0& 0.8\% & 1.2& 0.2\% & 0.8& 0.0\% & NA \\
  \blexec{20} & 44.4\% & 0.8& 19.9\% & 1.0& 6.3\% & 1.0& 2.4\% & 1.6& 0.8\% & 1.9 \\
  \blexec{30} & 58.9\% & 0.7& 34.0\% & 0.9& 18.8\% & 1.2& 9.2\% & 1.2& 4.5\% & 1.4 \\
  \blexec{40} & 68.8\% & 0.6& 52.2\% & 0.8& 34.7\% & 1.0& 23.8\% & 1.1& 11.8\% & 1.5 \\
  \blexec{50} & 75.0\% & 0.5& 62.2\% & 0.7& 49.0\% & 0.9& 37.8\% & 1.1& 27.9\% & 1.2 \\
  \blexec{60} & 78.8\% & 0.4& 69.0\% & 0.5& 59.4\% & 0.7& 53.4\% & 0.9& 44.5\% & 1.1 \\
\end{tabular}
  }
  \caption{Coverage and average number of re-plans in the duration-uncertain delivery domain.}
  \label{tab:results}
\end{table*}

\myparagraph{Duration-Uncertain Flexible Execution}
We use the Robot Delivery domain to investigate the merits of an on-line plan executor equipped with our IRR algorithm. In particular, we begin by focusing the analysis on the number of re-plans and on the plan execution success rate when only the duration of actions is uncertain during execution.
In this domain, a robot has to collect a spindle from a shelf, construct a base by performing six steps (possibly in parallel), then mount a number of rings, and finally deliver the order. Orders have deadlines that must be met for delivery.
The domain allows the agent to drop an order and restart from scratch with a new one at any time, but this disposal action takes some time (10 seconds in our case) and the robot needs to navigate on a symbolic euclidean graph to pick the parts, assemble and deliver the order.
Each instance is simulated in an environment where actions have a non-deterministic duration described by a normal distribution with a minimum value. Due to the difficulty in manipulation tasks, the actions executed for preparing the base (in which the robots interact with machines) have the highest degree of variance. These actions have mean durations of $120$, $130$, $140$, $150$, $160$ and $170$ seconds, and a standard deviation of $70$. Due to this uncertainty and the presence of deadlines for the order delivery, the execution of a plan can fail even when a re-planning schema is employed.
We generated a total of 100 problems by varying the deadlines for the orders.

Our DRE-based approach was implemented in ROSPlan, as described in section \ref{sec:exe}. The STN dispatcher starts the execution of actions following the temporal constraints of the STN: the process is illustrated in algorithm \ref{algo:dispatch}. For each node, the minimum and maximum dispatch times are calculated during execution (line 5). The dispatch ends when an action completes outside of the temporal constraints allowed by the STN, or has not been started after the maximum allowed dispatch time. When the dispatch ends, it returns $true$ if the goals have been achieved; otherwise, re-planning is triggered as shown in figure \ref{fig:flow}. The system will continuously attempt to re-plan until the deadlines make the PDDL planning problem unsolvable.

\begin{algorithm}[tb]
  \small
  \begin{algorithmic}[1]
    \Function{STNDispatch}{$\pi_{stn},\rho$}
    \State{$finished = false$}
  	\While{$\neg finished$}
    \ForEach{node $n \in \pi_{stn}$}
    \State{$min, max \gets$ \Call{MinMaxDispatchTime}{$n,\pi_{stn},\rho$}}
    \If{$n$ is action start}
        \If{$(min \leq n \leq max) \wedge \neg$\Call{started}{$n$}}
            \State{\Call{StartExecuting}{$n$}}
        \ElsIf{$(n \geq max) \wedge \neg$\Call{started}{$n$}}
            \State{$finished = true$}
        \EndIf
    \ElsIf{$n$ is action end}
        \If{$(n \geq max) \wedge \:\neg$\Call{Completed}{$n$}}
            \State{$finished = true$}
        \ElsIf{$(n \leq min) \wedge \:$\Call{Completed}{$n$}}
            \State{$finished = true$}
        \EndIf
    \EndIf
    \EndFor
    \EndWhile
  	\State{\Return{\Call{GoalsAchieved}{ }}}
  	\EndFunction
  \end{algorithmic}
  \caption{\label{algo:dispatch} STN Dispatch}

\end{algorithm}

We compare the executor described in section \ref{sec:exe} (indicated as \dreexec) against several baselines in which we dispatch the STN plan $\pi_{stn}$ without parameterization. In such baselines, the executor dispatches the STN plan allowing for a fixed deviation in the duration of actions and ends dispatch only when the action duration falls outside of this interval. This is the optimistic technique for execution implemented in ROSPlan that, differently from \dreexec, offers no formal guarantees. We consider baseline executors named \blexec{0} to \blexec{60} allowing for 0\% to 60\% variability in action duration before triggering a re-plan. For example, given an action with a predicted duration 100 seconds, \blexec{0} will re-plan if the duration is not exactly 100; \blexec{20} will re-plan if the duration is outside of the interval $[80,120]$. The baseline \blexec{0} corresponds to formally executing the time-triggered plan $\pi_{tt}$: re-planning happens if any action duration differs from what was expected in $\pi_{tt}$. We highlight that, when \dreexec is employed and the observation is within the envelope computed by IRR, we have the formal guarantee of plan success; as soon as one observation is outside of the envelope, we choose to re-plan.

The overarching idea in these experiments is that the planner usually optimistically selects the easier, quicker goal and the agent starts to execute the plan. If the execution of the preparation actions goes overlong, it might become impossible to deliver the order, so the only way to successfully recover is to immediately dispose the current order and switch to another one with a less imminent deadline. If the executor fails in realizing this situation, it continues to execute the plan until it tries to deliver the order, at which point it realizes that the deadline is not met. Since a lot of time has been wasted in the preparation, it might be impossible to recover from this situation.
Ideally, we expect that the predictive power of DREs allows the identification of situations where the deadline cannot be met and a swift re-planning to change the objective order is needed.

Table~\ref{tab:results} reports the results of the experiment. We report the coverage percentage (i.e. the percentage of problems successfully executed over the benchmark set) as well as the average number of re-plannings for successful runs. The baseline \blexec{0}, not accounting for any variance in action duration, was unable to solve any problem successfully. Allowing for more flexibility in the duration of actions increases the coverage as should be expected. However, the \dreexec approach achieves greater coverage than all baselines in all the cases. This is because in this problem, the ability to realize early that the agent is late for the first order and change course of actions to achieve the second order is pivotal for achieving a good success rate.

\myparagraph{Resource-Uncertain Flexible Execution}
Finally, we show that our flow can be used when parameters are not just action durations. We expanded the delivery domain to consider the battery consumption of the robots. In particular, each action in the revised domain checks that enough battery is present upon start and consumes a fixed amount of battery. We parametrized the consumption rate of actions, so that the DRE will compute the possible consumption values for which a given plan is valid. The executor is then demanded to observe the contingent consumption and possibly invoke a re-planning if the observation does not fall in the DRE prescription. Also in this case, the baselines \blexec{X} invoke the replanning when the battery consumption is observed to be $X\%$ higher or lower than the nominal value.

\begin{table}[tb]
  \centering
  \resizebox{.75\columnwidth}{!}{
  \setlength{\tabcolsep}{5pt}
\begin{tabular}{l|c|c}
  \bf Executor & \bf Coverage  & \bf Avg Replans \\
  \hline &&\\[-1.5ex]
  \dreexec & \bf 99.2\% & \bf 0.1\\
  \blexec{0} & 1.0\% & 2.0\\
  \blexec{10} & 25.6\% & \bf 0.1\\
  \blexec{20} & 50.7\% & \bf 0.1\\
  \blexec{30} & 66.4\% & \bf 0.1\\
  \blexec{40} & 77.9\% & \bf 0.1\\
  \blexec{50} & 82.1\% & \bf 0.1\\
  \blexec{60} & 86.8\% & \bf 0.1\\
\end{tabular}
  }
  \caption{Coverage and average number of re-plans in the resource-uncertain delivery domain.}
  \label{tab:res-results}

  \vspace{-0.7cm}

\end{table}

Table~\ref{tab:res-results} reports the results of the experiment, and
shows how the use of \dreexec is beneficial for the success-rate
achieving an almost perfect success-rate with very few replannings on
average.


\vspace{-0.8cm}

\section{Conclusion}

In this paper, we make the case for the use of REs in a plan execution framework. We present a novel, anytime algorithm to compute DREs that is empirically superior to the previously known logic-based construction. Moreover, we demonstrate the usefulness of the produced artifacts by integrating them in the ROSPlan framework and showing on a concrete example the positive impact on the number of re-plannings and the plan success-rate.

In the future, we will consider other kinds of approximations for the robustness envelope (e.g. hyper-octagons instead of hyper-rectangles). We will also explore the link to temporal uncontrollability and non-deterministic planning. Finally, using IRR in parallel with the dispatcher could allow variation in parameters being considered during execution.

\vspace{-0.2cm}
\section{Acknowledgements}

This work was partially supported by Innovate UK grant 133549: \textit{Intelligent Situational Awareness Platform}, and by EPSRC grant  EP/R033722/1: \textit{Trust in Human-Machine Partnerships}. Also, this work is partially supported by the Autonomous Province of Trento in the scope of L.P. n.6/1999 with grant MAIS (Mechanical Automation Integration System) n. 2017-D323-00056 del. n. 941 of 16/06/2017 and by EIT DIgital within the "AWARD" project.

\newpage
\bibliographystyle{named}
\bibliography{refs.bib}

\end{document}